\pdfoutput=1
\documentclass[10pt,a4paper]{article}
\usepackage[T1]{fontenc}
\usepackage[utf8]{inputenc}
\usepackage{lmodern,microtype}
\usepackage[margin=1in,includeheadfoot]{geometry}
\usepackage{amsmath,amssymb,amsthm}
\usepackage{graphicx,booktabs,tabularx,array}
\usepackage[font=normalsize,labelfont=bf]{caption}
\usepackage{xcolor,fancyhdr,titlesec,placeins}
\usepackage[numbers,sort&compress]{natbib}
\usepackage[hidelinks]{hyperref}
\usepackage{url}
\hypersetup{pdftitle={Approaching Backpropagation Accuracy with Weight Perturbation on MNIST},pdfauthor={Patrick Inoue},pdfsubject={Weight perturbation learning, training budgets, and finite-horizon gradient approximation},pdfkeywords={weight perturbation, node perturbation, backpropagation, MNIST, zeroth-order optimization}}
\titleformat{\section}{\large\bfseries}{\thesection}{0.7em}{}
\titleformat{\subsection}{\normalsize\bfseries}{\thesubsection}{0.7em}{}
\titlespacing*{\section}{0pt}{1.9ex plus .4ex minus .2ex}{.7ex plus .2ex}
\titlespacing*{\subsection}{0pt}{1.4ex plus .3ex minus .2ex}{.5ex plus .2ex}
\newcommand{\E}{\mathbb{E}}
\newcommand{\R}{\mathbb{R}}
\newcommand{\Cov}{\operatorname{Cov}}
\newcommand{\tr}{\operatorname{tr}}
\newtheorem{proposition}{Proposition}
\newcolumntype{Y}{>{\raggedright\arraybackslash}X}
\definecolor{curveblue}{HTML}{1F77B4}
\definecolor{curveorange}{HTML}{FF7F0E}
\definecolor{curvegreen}{HTML}{2CA02C}
\definecolor{curvered}{HTML}{D62728}
\definecolor{curvepurple}{HTML}{9467BD}
\newcommand{\key}[2]{\mbox{\textcolor{#1}{\rule[.55ex]{15pt}{1.05pt}}\enspace #2}}
\date{}
\title{\vspace{-1.4em}\LARGE\bfseries Approaching Backpropagation Accuracy\\with Weight Perturbation on MNIST}
\author{Patrick Inoue\thanks{The MNIST experiments originate from the author's master's thesis, submitted under the name Patrick Stricker at Albstadt-Sigmaringen University in March 2023~\cite{stricker2023}.}}

\begin{document}
\maketitle
\thispagestyle{plain}
\vspace{-2.0em}
\begin{abstract}\normalsize
\noindent Weight perturbation (WP) estimates learning directions from random parameter perturbations and a scalar change in loss, without backpropagating derivatives. We examine how batch size, training duration, and estimator stabilization affect its attainable accuracy in fully connected MNIST networks. In the shallow configuration, large-batch training reaches approximately 97\% test accuracy, compared with an early backpropagation (BP) reference near 97.5--98\%. The same comparison exposes a substantial efficiency gap: under a ten-epoch budget, shallow WP reaches 90.2\%, whereas the reported BP reference reaches 97.4\%. Epoch-indexed curves show diminishing returns from increasing batch size beyond 4,000. Learning-rate reduction and averaging independent perturbation-based updates also stabilize deeper networks, although their recorded results do not establish BP-level accuracy. We distinguish these empirical findings from a finite-horizon approximation argument: under stated regularity conditions, population WP can approximate a fixed gradient-descent trajectory arbitrarily closely at any fixed finite depth. Positive classification margins then suffice to preserve the reference predictions on a finite test set with arbitrarily high probability. Together, the results identify near-BP accuracy as a configuration- and budget-dependent possibility, rather than evidence of equal efficiency or an unconditional accuracy guarantee across depth.
\end{abstract}

\section{Introduction}

\noindent Learning from global performance feedback without transmitting error derivatives through neural circuitry is a central challenge in biologically plausible credit assignment. Backpropagation (BP) solves this problem computationally, but its conventional implementation relies on coordinated forward and backward computations, precise layer-specific error signals, and corresponding feedback pathways whose biological realization remains debated~\cite{rumelhart1986,lillicrap2020}. Biological synapses, in contrast, exhibit spontaneous fluctuations and remodeling, while synaptic plasticity can be regulated by neuromodulatory signals associated with reward and behavioral outcomes~\cite{seung2003,ziv2018,reynolds2002,kusmierz2017}. These properties motivate learning mechanisms that exploit local stochastic exploration and globally available evaluative signals to guide synaptic adaptation.

Weight perturbation (WP) offers a particularly direct approach to this problem. Synaptic weights are transiently perturbed, and the resulting change in task performance determines the direction of subsequent plasticity~\cite{dembo1990,cauwenberghs1993}. Each synapse requires only information about its own perturbation and a shared scalar feedback signal. The learning rule operates without an explicit differentiable model of the network, backward error propagation, or symmetric feedback weights. WP therefore provides a model-free candidate for biologically motivated learning, connecting stochastic synaptic plasticity with gradient-based optimization~\cite{seung2003,fiete2006,zuge2023}. Perturbation-based learning has also been combined with competitive Hebbian plasticity in sparse excitatory neural networks~\cite{stricker2024}, illustrating its compatibility with additional biologically motivated learning mechanisms.

The principal obstacle is the statistical efficiency of high-dimensional credit assignment. WP infers a parameter update from scalar performance differences, introducing substantial variance into the estimated learning direction. Consequently, convergence can be considerably slower than gradient-based training~\cite{werfel2005,zuge2023}. On MNIST, Z\"uge et al.\ reported a test accuracy of $92.3\pm0.2\%$ for WP, compared with $97.7\pm0.2\%$ for stochastic gradient descent, in a network with one hidden layer after 50,000 weight updates~\cite{zuge2023}. This gap raises a fundamental question: does stochastic synaptic credit assignment impose a limit on attainable classification accuracy, or can improved estimator stability and extended training overcome it?

Here, we show that extended large-batch training brings WP to approximately $97\%$ test accuracy on MNIST, approaching the early BP reference of $97.5$--$98\%$ in a fully connected network with one hidden layer. The remaining accuracy gap is approximately $0.5$--$1$ percentage point. Training efficiency remains substantially different: under a ten-epoch budget, WP achieves $90.2\%$, compared with $97.4\%$ for the reported BP reference. We further show that learning-rate reduction and averaging independent perturbation estimates can stabilize WP training in networks with five hidden layers. Finally, a finite-horizon analysis establishes conditions under which population WP can approximate a gradient-descent trajectory arbitrarily closely. These findings demonstrate that high classification accuracy is attainable through model-free, perturbation-based synaptic learning, while identifying the training conditions and computational costs that govern its practical performance~\cite{stricker2023}.

\subsection{Perturbation learning: foundations and recent developments}
\noindent Scalar-feedback learning has a substantial history in model-free learning, stochastic error descent, and statistical gradient-following algorithms~\cite{dembo1990,cauwenberghs1993,williams1992}. Simultaneous perturbation stochastic approximation provides a related finite-difference route to optimization~\cite{spall1992}, while early neural-network work connected weight and node perturbations and their computational costs~\cite{flower1992}. Conductance-perturbation models extended these ideas to spiking networks and birdsong learning~\cite{fiete2006,fiete2007}. Population averaging was subsequently used to reduce variability in perturbation-based learning~\cite{suszynski2013}.

The relative efficiency of WP and node perturbation (NP) is not determined by parameter count alone. Analyses in linear networks identify dimension-dependent learning costs~\cite{werfel2005}, whereas Z\"uge et al.\ show that trial duration and effective task dimensionality can make WP comparable or superior to NP~\cite{zuge2023}. The 2023 journal article succeeds the 2021 preprint used in the thesis and is reproduced with its supplement in Z\"uge's 2025 dissertation~\cite{zuge2025}. For NP, Hiratani et al.\ identify diffusion-driven instability under model mismatch and a stabilizing role for normalization~\cite{hiratani2022}. Dalm et al.\ relate NP to directional derivatives and combine this perspective with input decorrelation, reporting substantially improved multilayer learning in their revised 2026 preprint~\cite{dalm2026}.

For MNIST, Z\"uge et al.\ report that their fully trained $784\!\rightarrow\!100\!\rightarrow\!10$ network reaches $92.3\pm0.2\%$ with WP versus $97.7\pm0.2\%$ with SGD at $B=1{,}000$ after 50,000 updates (five-run means and standard deviations; supplementary Table~2)~\cite{zuge2023}. The no-hidden-layer control has a smaller gap at a lower accuracy (Fig.~S11); its teacher-target variants are different tasks. The dissertation retains these experiments, without adding a long-budget MNIST result that closes the hidden-layer gap~\cite{zuge2025}. Our study extends the batch-size and training-budget range; it does not establish superiority under a matched cross-study protocol.

Recent work includes the Dopamine optimizer for XOR and time-series tasks~\cite{nambusubramaniyan2025}. Gaussian zeroth-order optimization~\cite{nesterov2017}, population-based evolution strategies~\cite{salimans2017}, and forward-evaluation language-model fine-tuning~\cite{malladi2023} provide broader context, but not matched from-scratch MNIST benchmarks. Forward gradients use automatic differentiation~\cite{baydin2022} and are distinct from derivative-free loss-difference estimators.

\section{Materials and Methods}
\subsection{Task and architectures}

\noindent MNIST contains 60,000 training and 10,000 test images with ten digit classes~\cite{lecun1998}. Following the thesis~\cite{stricker2023}, images are flattened to 784 components, scaled to $[0.01,0.99]$, and processed by fully connected hidden layers with hyperbolic-tangent activation and biases, followed by a ten-unit softmax output. The shallow architecture is $784\!\rightarrow\!100\!\rightarrow\!10$. The deeper configuration, called Scen1 in the thesis, has five hidden layers of 100 units each. A wider five-hidden-layer configuration uses widths $2500,2000,1500,1000,500$, following the architecture family of Cire\c{s}an et al.~\cite{ciresan2010}, but without their image-deformation protocol.

Weights use the stated Xavier-normal initialization with variance $2/(m_{\mathrm{in}}+m_{\mathrm{out}})$~\cite{glorot2010,stricker2023}. The loss on a batch $\mathcal B$ of size $B$ is
\begin{equation}
 E_{\mathcal B}(\theta)=-\frac1B\sum_{n\in\mathcal B}\sum_{k=1}^{10}
 \widetilde y_{nk}\log p_{\theta,k}(x_n),
 \label{eq:loss}
\end{equation}
where $\widetilde y$ denotes the label-smoothed target distribution. The batch-mean convention follows the thesis's methods prose and appendix; its displayed loss equation omits $1/B$. Reported learning rates are therefore retained without rescaling.

\subsection{Weight perturbation and population averaging}

\noindent Let $\theta\in\R^d$ collect the perturbed parameters. At each update, draw $\xi\sim\mathcal N(0,\sigma^2I_d)$ and retain this perturbation for every example in the batch. Both loss evaluations use the same inputs and targets. The baseline-subtracted estimator and update are
\begin{equation}
 \widehat g_{\sigma,1}(\theta;\mathcal B)
 =\frac{E_{\mathcal B}(\theta+\xi)-E_{\mathcal B}(\theta)}{\sigma^2}\,\xi,
 \qquad \theta_{t+1}=\theta_t-\eta_t\widehat g_{\sigma,1}.
 \label{eq:wp}
\end{equation}
\noindent Perturbed parameters are temporary; the update is applied to the unperturbed vector. A lower perturbed loss reinforces the perturbation direction. This is the simultaneous Gaussian WP rule used in the thesis and its reference study~\cite{stricker2023,zuge2023}.

Population WP averages complete loss--perturbation products from $P$ independent perturbations at the same parameter vector~\cite{suszynski2013,stricker2023}:
\begin{equation}
 \widehat g_{\sigma,P}(\theta;\mathcal B)=\frac1P\sum_{p=1}^P
 \frac{E_{\mathcal B}(\theta+\xi^{(p)})-E_{\mathcal B}(\theta)}{\sigma^2}\,\xi^{(p)}.
 \label{eq:population}
\end{equation}
\noindent The baseline is shared, and population members are not retained as an inference ensemble. This is an average of products, not a product of separately averaged losses and perturbations. Independence gives the exact conditional identity
\begin{equation}
 \Cov_{\xi}(\widehat g_{\sigma,P}\mid\theta,\mathcal B)
 =\frac1P\Cov_{\xi}(\widehat g_{\sigma,1}\mid\theta,\mathcal B).
 \label{eq:variance}
\end{equation}
\noindent Increasing $B$ and increasing $P$ therefore act on different sources of variability: batch composition and perturbation realization. Population averaging reduces estimator variance; it does not remove disturbances before their propagation through the network. 

\subsection{Node perturbation and comparison protocol}
\noindent NP perturbs neuronal inputs or activities rather than every parameter. In the temporally extended formulation motivating the experiments, node perturbations vary across examples, and their correlations with presynaptic activity are accumulated before multiplication by a scalar loss difference~\cite{zuge2023}. WP instead holds one parameter perturbation fixed over the batch. NP supplies a reference mechanism for understanding this distinction; no additional NP experiments are introduced here.

The shallow reference uses up to approximately 50,000 updates. The update-indexed extended comparisons display up to approximately 450,000 updates. The large-batch WP--BP comparison uses batch sizes from 2,000 to 10,000. The separate budget comparison also includes $B=1{,}000$ and displays up to 22,500 epochs in its epoch-indexed panel. A separate ten-epoch comparison uses locally selected settings from Random Grid Path Search~\cite{knoblauch2022,stricker2023}; these settings are not imputed to the long runs.

For $U$ updates with a fixed batch size, nominal data exposure is
\begin{equation}
 N_{\mathrm{epochs}}=\frac{UB}{60{,}000}.
 \label{eq:epochs}
\end{equation}
\noindent This counts batch presentations, not all forward evaluations. Population WP requires $P+1$ forward evaluations per update, versus two for $P=1$. BP additionally computes derivatives. Thus equal update counts, equal epochs, and equal runtime are distinct comparison criteria.

\section{Results}
\subsection{From the small-batch gap to near-BP accuracy}
\noindent The starting gap depends strongly on configuration (Fig.~\ref{fig:reference}). At approximately 50,000 updates, shallow WP reaches about 63\% accuracy with $B=1$, compared with about 95\% for BP. With $B=1{,}000$, WP reaches approximately 92--93\%, against BP near 98\%. Equal updates do not equalize data exposure: these settings correspond to roughly 0.83 and 833 nominal epochs, respectively.

\begin{figure}[tbp]
\centering
{\normalsize\key{curveblue}{WP}\hspace{2em}\key{curveorange}{BP}\par}\vspace{4pt}
\includegraphics[width=.488\linewidth]{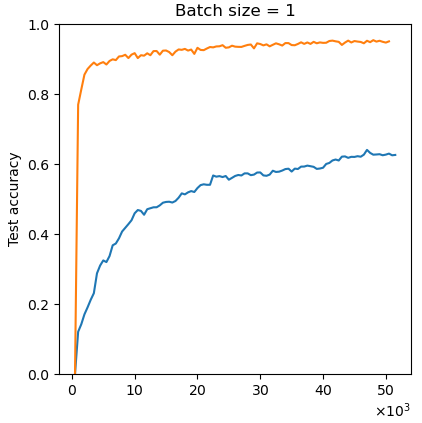}\hfill
\includegraphics[width=.488\linewidth]{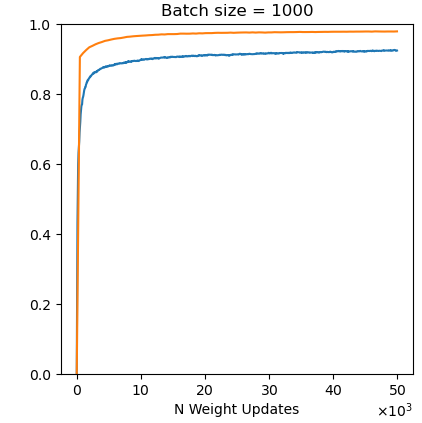}
\caption{Shallow reference from the author's thesis~\cite{stricker2023}, reproducing the setting of Z\"uge et al.~\cite{zuge2023}. Left: batch size 1; right: 1,000. Blue denotes WP and orange BP. Both horizontal axes count weight updates in thousands. The original curves and axes are unchanged.}
\label{fig:reference}
\end{figure}

The extended comparison is the principal accuracy result (Fig.~\ref{fig:long}). For batches of 4,000 and 10,000, shallow WP rises to approximately 96.8\% and 97.0\%, respectively, near the end of the displayed 450,000-update budget. The early BP reference is approximately 97.5--98\%, leaving a visually estimated gap of about 0.5--1 percentage point. These curves directly support near-BP performance for this shallow configuration.

\begin{figure}[tbp]
\centering
{\normalsize\key{curveblue}{WP, shallow}\hspace{1.4em}\key{curveorange}{BP}\hspace{1.4em}\key{curvegreen}{WP, five hidden layers}\par}\vspace{4pt}
\includegraphics[width=.487\linewidth]{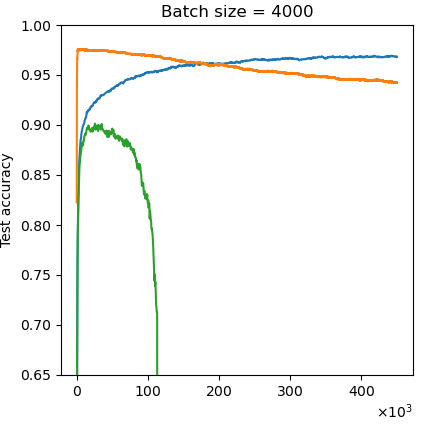}\hfill
\includegraphics[width=.487\linewidth]{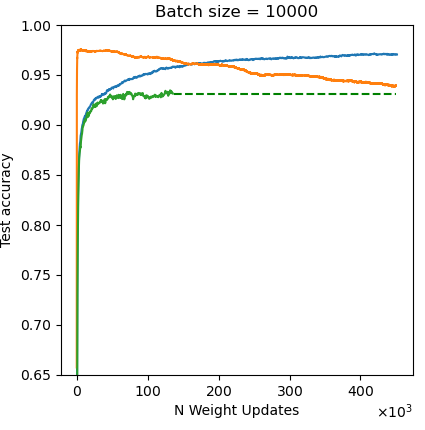}
\caption{Extended training at batch sizes 4,000 (left) and 10,000 (right), selected from the original four-panel thesis figure~\cite{stricker2023}. Blue is the shallow WP model (original label ``Baseline''), orange the BP reference, and green five-hidden-layer WP (``Scen1''). The green dashed segment at $B=10{,}000$ extends the last shown level; it is not a measured continuation. Axes and solid curves are preserved. The truncated vertical range starts at 0.65. The relevant accuracy comparison is with BP's early peak, not merely its deteriorated final checkpoint.}
\label{fig:long}
\end{figure}

BP reaches its high accuracy early and subsequently declines towards roughly 94\% during prolonged training. WP eventually exceeds these late BP checkpoints, but this does not establish superiority over a suitably selected BP model. The thesis itself notes that such prolonged BP training was included for comparison with WP and was not a recommended BP training protocol~\cite[Sec.~3.3]{stricker2023}. The empirical conclusion is proximity to the earlier BP accuracy at a much larger training cost.

\subsection{Batch size, data exposure, and learning speed}
\noindent Figure~\ref{fig:exposure} separates the two budget views reported in the thesis. Against update count, batches of 4,000--10,000 approach a similar accuracy near 97\%, while batches of 1,000 and 2,000 remain lower. Against epochs, the largest batch no longer gives the strongest curve: $B=4{,}000$ reaches a higher accuracy within the displayed exposure range than $B=10{,}000$. Beyond a useful stabilization regime, processing more examples per update yields diminishing returns.

\begin{figure}[p]
\centering
{\normalsize\key{curveblue}{$B=1{,}000$}\hspace{1.1em}\key{curveorange}{$B=2{,}000$}\hspace{1.1em}\key{curvegreen}{$B=4{,}000$}\par}
\vspace{3pt}
{\normalsize\key{curvered}{$B=6{,}000$}\hspace{1.4em}\key{curvepurple}{$B=10{,}000$}\par}\vspace{7pt}
{\normalsize\textbf{(a)}\; Nominal training epochs\par}\vspace{2pt}
\includegraphics[width=.79\linewidth]{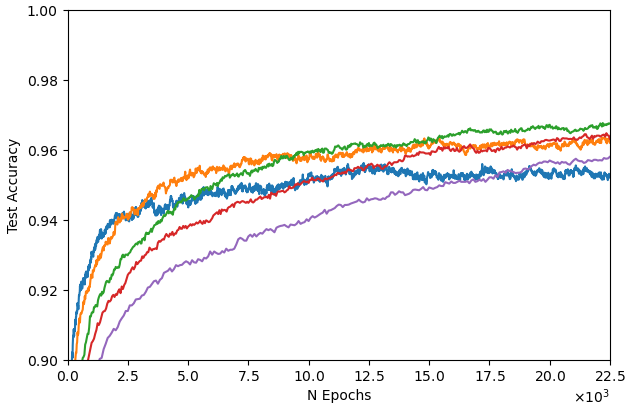}\par\vspace{8pt}
{\normalsize\textbf{(b)}\; Weight updates\par}\vspace{2pt}
\includegraphics[width=.79\linewidth]{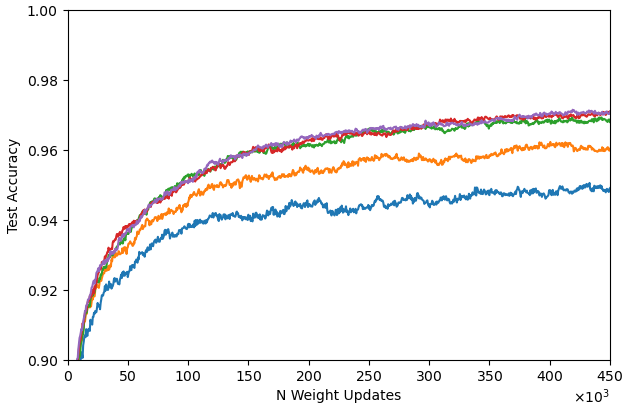}
\caption{Budget dependence of shallow WP, preserving the two original panels from the thesis~\cite{stricker2023}. Colors encode batch size as in the reinstated legend: 1,000 (blue), 2,000 (orange), 4,000 (green), 6,000 (red), and 10,000 (purple). Top: epochs; bottom: weight updates. Both horizontal axes carry a $10^3$ multiplier, and both vertical ranges start at 0.90.}
\label{fig:exposure}
\end{figure}

The distinction is substantial: 450,000 updates at $B=4{,}000$ represent 30,000 nominal epochs, and at $B=10{,}000$ they represent 75,000. The ten-epoch results give a complementary short-budget comparison (Table~\ref{tab:short}). Shallow WP reaches 90.2\%, 7.2 percentage points below the reported BP reference. Five-hidden-layer WP reaches 73.1\%. Near-BP long-run accuracy consequently does not imply near-BP learning speed.

\begin{table}[tbp]
\centering
\caption{Ten-epoch results transcribed from the thesis~\cite[Sec.~3.3]{stricker2023}. These settings are specific to the short-budget search. The WP perturbation-strength setting is reported as $1.6\times10^{-3}$.}
\label{tab:short}
\normalsize
\begin{tabularx}{\linewidth}{@{}Yrrr@{}}
\toprule
Configuration & Batch size & Learning rate & Accuracy (\%)\\
\midrule
WP, shallow & 2,000 & $2.15\times10^{-3}$ & 90.2\\
WP, five hidden layers & 2,000 & $4.00\times10^{-3}$ & 73.1\\
BP, reported reference & 4,000 & $1.50\times10^{-2}$ & 97.4\\
\bottomrule
\end{tabularx}
\end{table}

\FloatBarrier
\subsection{Stabilization of deeper networks}
\noindent The five-hidden-layer comparison exposes a stronger stability constraint. At $B=4{,}000$, Scen1 initially improves but subsequently collapses below the plotted range. At $B=10{,}000$, its solid curve reaches approximately 93\% before the dashed extension begins. Larger batches can improve trainability without establishing deep-network parity.

The thesis further reports that decreasing the learning rate stabilizes narrow deep networks. Population averaging provides a complementary intervention: ten perturbations in Scen1 yield accuracy comparable to a reduced-learning-rate reference, while four perturbations stabilize the wide five-hidden-layer configuration where preceding single-perturbation settings failed~\cite[Sec.~3.4.3]{stricker2023}. Indeed, the thesis's final comparison states that the deeper reference had not achieved BP-comparable accuracy within the considered budgets~\cite[Sec.~3.6]{stricker2023}.

\section{Discussion}

\noindent WP can attain near-BP accuracy on MNIST in a fully connected network trained through scalar loss feedback. Extended large-batch training reaches approximately $97\%$ test accuracy, within about $0.5$--$1$ percentage point of the early BP reference (Fig.~\ref{fig:long}). This result demonstrates that the lower accuracies observed in the small-batch runs are not an intrinsic ceiling of the WP learning rule. A substantial part of the accuracy deficit can be recovered through stabilization and additional training. The experiments therefore separate two properties that are easily conflated when evaluating biologically motivated learning: attainable accuracy and learning efficiency.

The high-accuracy regime extends the earlier MNIST comparison of Z\"uge et al.~\cite{zuge2023}. Their network with one hidden layer reached $92.3\%$ with WP and $97.7\%$ with SGD after 50,000 updates at a batch size of 1,000. Our extended experiments show that WP can move substantially closer to the BP reference when larger batches and longer training are available. The contribution is this empirical demonstration of near-BP accuracy in the hidden-layer setting. It complements the established interpretation of WP as a stochastic gradient estimator~\cite{dembo1990,cauwenberghs1993,werfel2005} by showing how much performance can be recovered in a nonlinear classification network.

Computational efficiency remains the principal practical limitation. In the ten-epoch comparison, WP reaches $90.2\%$, whereas the reported BP reference reaches $97.4\%$ (Table~\ref{tab:short}). Approaching the latter accuracy with WP requires substantially more training. This behavior is consistent with the variance introduced when a high-dimensional update is estimated from a scalar loss difference~\cite{werfel2005,zuge2023}. Smaller learning rates limit destabilizing parameter changes, while averaging independent perturbation estimates reduces update variability~\cite{suszynski2013}. These mechanisms support continued learning, but they incur costs through additional updates or additional forward evaluations.

The batch-size comparison further shows why more averaging is not automatically more efficient. Larger batches expose each perturbation to a greater fraction of the task, while population averaging combines independent perturbation realizations. In the update-indexed curves, batches from 4,000 to 10,000 approach similar accuracies. When performance is indexed by epochs, the largest batch loses its advantage (Fig.~\ref{fig:exposure}). The useful operating regime therefore balances update stability against the amount of computation devoted to each step. Maximizing batch size or population size alone can yield diminishing returns; the objective is sustained learning with an affordable amount of averaging.

Network depth makes this balance more demanding. In the tested deeper architectures, settings that support shallow learning can produce unstable updates, and recovering useful learning requires smaller learning rates or stronger averaging. These interventions improve trainability, although the deeper configurations remain below the BP comparison within the investigated budgets. The results identify stabilization and slow convergence as central challenges for scaling WP. A larger representation hierarchy increases the demands on the learning process, so architectural capacity and effective training progress must be considered together. Improving the information extracted from each perturbation is consequently an important direction for extending WP to deeper networks.

Hybrid learning provides a complementary direction. Stricker et al.~\cite{stricker2024} combine competitive Hebbian feature learning with a joint Hebbian--WP classification-layer update in sparse excitatory networks. Inoue et al.~\cite{inoue2026} extend this framework to deeper nonnegative architectures using a PCA-related competitive update and inter-layer normalization. In these systems, hidden representations develop through local activity-dependent learning, while perturbation-based supervision acts at the readout. The present study addresses WP across all weight matrices. Together, these approaches motivate examining how the allocation of learning between local representation mechanisms and scalar-feedback updates affects accuracy, stability, and computational cost.

The empirical scope of the present results is fully connected MNIST classification. Extending the comparison to more complex tasks will determine how the required training budget changes with problem difficulty and network size. Relating the large exposure budgets to biologically feasible learning times is also an important question. The immediate methodological priority is to reduce the cost of reaching the high-accuracy regime demonstrated here, through more informative perturbation estimates and a better balance between learning rate, batch size, and population averaging.

\subsection{Conclusion}

\noindent Model-free weight perturbation reaches approximately $97\%$ accuracy on MNIST and approaches the performance of backpropagation. This demonstrates that a biologically motivated learning rule can support high classification accuracy despite substantially slower learning. Larger batches, longer training, and stable updates are decisive in recovering that performance. The accompanying finite-horizon analysis establishes conditions for arbitrarily close approximation of a fixed gradient-descent trajectory. For the configurations examined here, the main practical challenge is the efficiency with which WP reaches high accuracy, especially as network depth increases.

\section*{Acknowledgements}
The author thanks Andreas Knoblauch and German Nemirovski for supervising the master's thesis underlying this study. The original experiments used the bwUniCluster~2.0 computing infrastructure, as documented in the thesis~\cite{stricker2023}.

\section*{Prior dissemination and availability}
The empirical results and original learning-curve artwork derive from the author's 2023 master's thesis~\cite{stricker2023}, authored as Patrick Stricker and referenced through its OPUS-HSAS record (document 1713). This article selects and reframes its MNIST findings and adds the explicit finite-horizon analysis and updated literature context; the experiments are not presented as newly performed. The figures are reproduced as selected panels with restored legends and clarified captions. The underlying run logs and training code are not distributed with this article.

\section*{Generative AI Statement}
The author declares that generative AI (OpenAI ChatGPT) was used in the preparation of this manuscript for initial drafting, language editing, proofreading, and assistance with the formulation of the finite-horizon analysis. The original master's thesis (Stricker, 2023), including all underlying experiments and results, was conducted and written entirely without generative AI. The author assumes full responsibility for the scientific content of this manuscript.

\appendix
\section{Finite-horizon approximation and classification agreement}
\label{app:tracking}
\subsection{Estimator bias and variance}

For a fixed batch, write $F(\theta)=E_{\mathcal B}(\theta)$ and $u\sim\mathcal N(0,I_d)$. Under conditions permitting Gaussian integration by parts,
\begin{equation}
 \E\widehat g_{\sigma,P}(\theta)=\nabla F_\sigma(\theta),
 \qquad F_\sigma(\theta)=\E_u F(\theta+\sigma u).
 \label{eq:smoothing}
\end{equation}
This Gaussian-smoothing identity underlies randomized gradient-free methods~\cite{nesterov2017,salimans2017}; the small-perturbation gradient interpretation of WP is already established~\cite{zuge2023}. The baseline term vanishes in expectation, while $\sigma^{-1}\E[F(\theta+\sigma u)u]=\E\nabla F(\theta+\sigma u)$. Thus finite-$\sigma$ WP is unbiased for the smoothed gradient, not generally the original gradient.

Sufficient smoothness gives bias $O(\sigma^2)$, and independence gives
\begin{equation}
 \E\|\widehat g_{\sigma,P}-\nabla F\|^2
 =\|\nabla F_\sigma-\nabla F\|^2+
 \frac{\tr\Cov(\widehat g_{\sigma,1})}{P}.
 \label{eq:mse}
\end{equation}
Increasing $P$ removes variance, not smoothing bias. Decreasing $\sigma$ alone does not remove the directional variability of a single estimate. In the linearized limit, writing $g=\nabla F(\theta)$, Gaussian fourth moments give
\begin{equation}
 \Cov\big((g^\top u)u\big)=\|g\|^2I_d+gg^\top,
 \qquad
 \E\left\|\frac1P\sum_{p=1}^P(g^\top u_p)u_p-g\right\|^2
 =\frac{d+1}{P}\|g\|^2.
 \label{eq:linearvariance}
\end{equation}
Consistency is therefore compatible with substantial cost in high dimension.

\subsection{Finite-horizon trajectory approximation}

The following statement formalizes the theoretical approximation claim; it is not an additional empirical result.

\begin{proposition}[Tracking a fixed gradient-descent reference]
Fix a finite-dimensional network, an initial vector, a finite sequence of batches, and $T<\infty$ learning rates $\eta_t\geq0$. Let
\begin{equation}
 \theta^*_{t+1}=\theta^*_t-\eta_t\nabla F_t(\theta^*_t),
 \qquad t=0,\ldots,T-1,
 \label{eq:reference}
\end{equation}
where $F_t$ is the prescribed batch loss. On a compact neighborhood of the reference trajectory, assume Lipschitz gradients and the uniform bounds
\begin{equation}
 \|\E\widehat g_{\sigma,P,t}(\theta)-\nabla F_t(\theta)\|
 \leq C_b\sigma^2,
 \qquad
 \tr\Cov(\widehat g_{\sigma,1,t}(\theta))\leq V
 \label{eq:uniform}
\end{equation}
for sufficiently small $\sigma$, with finite $C_b,V$. Assume deterministic loss evaluation, independent Gaussian perturbations, and exact arithmetic. For any $\varepsilon,\delta>0$, there exist $\sigma>0$ and finite $P$ such that WP, using the same initialization, batches, and learning rates, satisfies
\[
 \Pr\!\left(\max_{0\leq t\leq T}\|\theta_t-\theta_t^*\|<\varepsilon\right)>1-\delta.
\]
If the reference classifier at $T$ has a unique largest logit on every example of a finite test set and logits are continuous in the parameters, $\sigma$ and $P$ can also be chosen to yield identical test predictions with probability greater than $1-\delta$.
\end{proposition}

\begin{proof}
Let $L$ be a common Lipschitz constant. Conditional on the past and remaining in the neighborhood, Eq.~\eqref{eq:mse} bounds mean-square estimation error by $C_b^2\sigma^4+V/P$. Markov's inequality bounds the probability of an error exceeding $a>0$ at an update by $(C_b^2\sigma^4+V/P)/a^2$. Whenever all estimation errors are at most $a$, the parameter discrepancy $e_t=\|\theta_t-\theta_t^*\|$ obeys
\begin{equation}
 e_{t+1}\leq(1+\eta_t L)e_t+\eta_t a,
 \quad
 e_t\leq a\sum_{s=0}^{t-1}\eta_s
 \prod_{r=s+1}^{t-1}(1+\eta_rL).
 \label{eq:trackingbound}
\end{equation}
Choose $a$ so this finite-horizon bound is below both $\varepsilon$ and the neighborhood radius. Choose $\sigma$ small and $P$ large enough that $T(C_b^2\sigma^4+V/P)/a^2<\delta$. A conditional union bound up to the first possible exit, together with Eq.~\eqref{eq:trackingbound}, establishes the trajectory statement.

For classification, let $c_n$ be each reference prediction. Its minimum logit margin is
\[
 m=\min_n\left[z_{c_n}(x_n;\theta_T^*)-\max_{k\neq c_n}z_k(x_n;\theta_T^*)\right]>0.
\]
Finiteness and continuity give a parameter neighborhood in which every logit changes by less than $m/3$. All predictions are preserved there. Choose $\varepsilon$ within this neighborhood and apply the trajectory result.
\end{proof}

The reference includes mini-batch gradient descent with gradients computed by BP, but not arbitrary optimizers unless their state is also tracked. Constants may grow with depth, dimension, and horizon. The statement asserts neither global optimality nor equivalence at an affordable budget. Finite-precision cancellation and noisy loss evaluation can prevent arbitrarily small practical perturbations.

\begingroup
\normalsize
\setlength{\bibsep}{2.4pt}

\endgroup
\end{document}